\documentclass[twoside,leqno,twocolumn]{article}
\usepackage[numbers]{natbib}
\usepackage{ltexpprt}
\usepackage[utf8]{inputenc} 
\usepackage[T1]{fontenc}    
\usepackage[hidelinks]{hyperref}       
\usepackage{url}            
\usepackage{booktabs}       
\usepackage{nicefrac}       
\usepackage{microtype}      
\usepackage{subcaption}
\usepackage{epstopdf}
\usepackage{amssymb,relsize}
\usepackage[intlimits]{amsmath}
\usepackage{array}
\usepackage{mathabx}
\usepackage{appendix}
\usepackage{placeins}
\usepackage{enumitem}
\usepackage{subfiles}
\usepackage{caption}
\usepackage{graphicx}
\usepackage[percent]{overpic}
\usepackage{cleveref}
\usepackage{color}

\newcommand{\mat}[1]{\boldsymbol{\mathbf{#1}}} 
\newcommand{\inR}[1]{\in \mathbb{R}^{#1}}
\newcommand{\K}{\mat{K}}
\newcommand{\Kxx}{\K_\text{X,X}}
\newcommand{\Ksi}{\mat{Q}\left(\mat{T}+\sigma^2 \I{M}\right)^{-1}\mat{Q}^T} 
\newcommand{\Q}{\mat{Q}}

\newcommand{\T}{\mat{T}}

\newcommand{\Ki}{[\K^\text{-1}]}
\newcommand{\alp}{\boldsymbol \alpha}
\newcommand{\I}[1]{\mat{I}_{#1}}
\newcommand{\x}{\mat{x}}
\newcommand{\y}{\mat{y}}
\newcommand{\V}{\mat{V}}
\newcommand{\W}{\mat{W}}
\newcommand{\0}{\mat{0}}

\newcommand{\order}[1]{$\mathcal{O}(#1)$}
\newcommand{\nystrom}{Nystr\"{o}m}

  {\algorithm\renewcommand{\thealgorithm}{#1}}
  {\endalgorithm}
\newcommand{\dobib}{ 
	\bibliographystyle{IEEEtranN}
	\bibliography{core,peripheral}
}
\newcommand{\conditionalBib}{\dobib}

\newcommand{\lesslines}{\looseness=-1} 

\setlist[itemize]{leftmargin=*}

\newcommand{\lessfigspace}{-7mm}
\begin{document}
\renewcommand{\conditionalBib}{} 

\title{Exploiting Structure for Fast Kernel Learning}
\author{Trefor W. Evans \\ University of Toronto\\ \href{mailto:trefor.evans@mail.utoronto.ca}{\texttt{trefor.evans@mail.utoronto.ca}}
	\and
	Prasanth B. Nair\\ University of Toronto\\ \href{mailto:pbn@utias.utoronto.ca}{\texttt{pbn@utias.utoronto.ca}}}
\date{}

\maketitle


\fancyfoot[R]{\footnotesize{\textbf{Copyright \textcopyright\ 2018 by SIAM\\
			Unauthorized reproduction of this article is prohibited}}}



\begin{abstract}
We propose two methods for exact Gaussian process (GP) inference and learning on massive image, video, spatial-temporal, or multi-output datasets 
with missing values (or ``gaps'') in the observed responses.
The first method ignores the gaps using sparse selection matrices and 
a highly effective low-rank preconditioner is introduced to accelerate computations.
The second method introduces a novel approach to GP training whereby response values are inferred on the gaps \emph{before} explicitly training the model.
We find this second approach to be greatly advantageous for the class of problems considered.
Both of these novel approaches make extensive use of Kronecker matrix algebra to design massively scalable algorithms which have low memory requirements. 
We demonstrate exact GP inference for a spatial-temporal climate modelling problem with 3.7 million training points as well as a video reconstruction problem with 1 billion points.
\end{abstract}

\section{Introduction}
We introduce techniques to perform \emph{exact} Gaussian process (GP) training and inference on massive datasets by exploiting a structure that is present in many problems.
GP modelling is a powerful non-parametric framework for performing classification and regression, however,
exact GPs are typically restricted to small datasets since they require
\order{N^3} time and \order{N^2} storage where $N$ is the number of training points~\citep{rasmussen_gpml}.
This has motivated a considerable amount of work on scalable \emph{approximate} GP methods that can be applied to large-scale datasets
\citep{smola_sor,snelson_fitc,williams_nystrom,wilson_kiss,quin_ssgp} and even though significant progress has been
made on this topic, current methods
cannot generally achieve significant gains in scalability without a noticeable deterioration in accuracy.

\citet{saatci_phd} introduced scalable GP modelling techniques for datasets whose inputs are distributed
on a full Cartesian tensor product grid to
leverage efficient Kronecker matrix algebra~\cite{van_loan_kron}.
Training datasets structured in this particular form 
arise in many important applications including the analysis of
images, videos, spatial-temporal fields, sensor networks, or multi-output processes~\cite{osborne_sensor_networks, lawrence_multi_output_kernels}.
In most of these applications, there will be gaps in the training dataset
which may be caused by
missing observations,
presence of obstructions or irregular domain boundaries, or
data corruption~\citep{gunes_gappy, wilson_gpatt_nips}. 
Unfortunately, the efficient Kronecker matrix algebra used by \citet{saatci_phd} can no longer be used in the presence of these gaps.

\citet{wilson_gpatt_nips} introduce an extension to deal with gaps in structured data using a penalty method which works well
provided a suitable choice is made for a free penalty parameter.
In the present work, we provide alternative formulations that eliminate the need for a free parameter and
we demonstrate significant empirical speed improvements, particularly on massive datasets. 

We will consider datasets with $N$ training points which form a subset of $M$ points on a full Cartesian product grid.
Also, we denote the number of missing points (or gaps) from the full grid to be $L {=} M{-}N$.
Two GP formulations are developed which enable fast training and inference on a dataset with this partial grid structure.
While we restrict our discussion to GP models, the methods proposed in this paper can be readily applied to other kernel methods such as regularization networks or support vector machines~\cite{evgeniou}.
Here we summarize the main contributions of the paper:
\begin{itemize}
\item The proposed algorithms perform exact GP training and inference in
\order{dM^\frac{d+1}{d} {+} N} or \order{dM^\frac{d+1}{d} {+} L} time,
and \order{dM^\frac{2}{d}} storage; a significant improvement over standard GP models.
We also demonstrate marked improvements in both speed and robustness of the
proposed algorithms in comparison to
the ``penalty'' method introduced by \citet{wilson_gpatt_nips}.
\item A novel covariance matrix preconditioner is derived which we show significantly accelerates convergence. 
Direct application of this preconditioner to structured kernel interpolation (SKI) \cite{wilson_kiss} is also discussed.
%
%
\item We take a new approach to GP training where we infer the posterior mean on the gaps before training is complete.
We show that this technique is greatly advantageous for the considered class of problems.
\item For fast predictions, we demonstrate how the posterior mean can be efficiently computed on a grid and
we discuss how to quickly compute the posterior covariance exactly, or approximately.
\item Finally, we show that our methods are highly scalable and accurate; we train exact GPs on a massive video reconstruction problem and
we introduce a powerful approach to climate analysis in 
the reconstruction of daily temperatures at 291 Ontario weather stations over 56 years.
The largest problem contains over 1 billion points; to the best of our knowledge exact GP inference has not been attempted before on this scale.
\end{itemize}

\section{Background}
\label{sec:background}
We consider the set of inputs on a full Cartesian product grid
$\mat{\mathcal{X}} {=} \{\mat{x}_i\}_{i=1}^M$ 
where $\mat{x}_i \inR{d}$ is the $d$-dimensional input vector of the $i^\text{th}$ of $M$ points on the full grid.
The $M$ corresponding responses are denoted
$\mat{y} {=} \{y_i\}_{i=1}^M$ and
we assume that some of these responses are missing from our dataset.
Points in $\mat{\mathcal{X}}$ with known responses will be differentiated from points with missing responses through two index sets;
the set $\text{X}$ contains the indices of the $N$ training points with \emph{known} responses on the grid, and 
the set $\text{Z}$ contains the indices of the \emph{missing} training points on the grid such that
$\text{X} \bigcup \text{Z} =  \{i\}_{i=1}^M$ is the indices of \emph{all} $M$ points on the full grid.
When we write index sets in the subscript, we refer to a partition, 
i.e. we can write $\alp \inR{M}$, and $\K \inR{M\times M}$ in partitioned form as
\begin{equation*}
\alp = 
\left[ \begin{array}{c}
\alp_\text{X}\\
\alp_\text{Z}
\end{array}\right]
\quad \quad
\K = 
\left[ \begin{array}{cc}
\K_\text{X,X} & \K_\text{X,Z}\\
\K_\text{Z,X} & \K_\text{Z,Z}
\end{array}\right].
\end{equation*}
We will employ Gaussian processes (GPs) as non-parametric prior distributions over the latent function that generated the training dataset. 
It is assumed that the dataset is corrupted by independent Gaussian noise with variance $\sigma^2$
and that the latent function is drawn from a Gaussian process with zero mean and covariance determined by the kernel $k$.
Using the notation defined above, the log marginal likelihood of the targets with known response, $\mat{y}_\text{X}$, can be written as~\cite{rasmussen_gpml}
\begin{multline} \label{eqn:likelihood}
\log \mathcal{P}(\mat{y}_\text{X} | \mat{\theta}, \sigma^2, \mat{\mathcal{X}}_\text{X}) = 
-\tfrac{1}{2} \log | \Kxx + \sigma^2 \I{N} |
-\\ \tfrac{1}{2} \mat{y}_\text{X}^T (\Kxx + \sigma^2 \I{N})^{-1} \mat{y}_\text{X} 
-\tfrac{N}{2} \log(2\pi),
\end{multline}
where 
$\mat{\mathcal{X}}_\text{X}$ is the set of $N$ training point input positions,
and
$\Kxx \inR{N \times N}$ is the kernel covariance matrix evaluated on the training dataset which is a partition of $\K \inR{M \times M}$, the covariance matrix evaluated on the full tensor product grid; $[\K]_{i,j} = k(\x_i,\x_j)$.
The kernel is parameterized by the hyperparameters, $\mat{\theta}$, which we estimate by maximizing the marginal likelihood.

After training, inference can be carried out at an untried point, $\mat{x}_* \inR{d}$, giving the posterior distribution
\begin{multline} \label{eqn:posterior}
y_*| \mat{\theta}, \sigma^2, \mat{\mathcal{X}}_\text{X}, \mat{x}_* \sim \mathcal{N}\left(\mathbb{E}[y_*],\ \mathbb{V}[y_*]\right),\\ 
\begin{split}
\mathbb{E}[y_*] &= \mat{g}_\text{X}^T (\Kxx + \sigma^2 \I{N})^{-1} \mat{y}_\text{X},\\
\mathbb{V}[y_*] &= k(\mat{x}_{*},\mat{x}_{*}) - \mat{g}_\text{X}^T(\Kxx + \sigma^2 \I{N})^{-1}\mat{g}_\text{X},
\end{split}
\end{multline} 
where 
$y_* \inR{}$ is the test prediction and
$\mat{g} \inR{M}$ is the cross-covariance vector between all points on the full training grid and the test point, 
$[\mat{g}]_i = k(\mat{x}_i, \mat{x}_*)$.

We focus on the three computationally demanding calculations required for GP training and inference; 
\emph{i)}~solving the linear system 
$\alp_\text{X} = (\Kxx + \sigma^2 \I{N})^{-1} \mat{y}_\text{X}$ to compute the log marginal likelihood and posterior mean; 
\emph{ii)}~solving the linear system $(\Kxx + \sigma^2 \I{N})^{-1}\mat{g}_\text{X}$ for each test point to compute the posterior covariance;~and
\emph{iii)}~computing $\log | \Kxx + \sigma^2 \I{N} |$ for the log~likelihood.

We will consider a covariance kernel that obeys the product correlation rule (as many popular multidimensional kernels do), i.e.
$k(\x_i,\x_j) = \prod_{l=1}^{d} k_l(x_{il},x_{jl})$, 
in which case the covariance between points on a full tensor product grid inherits a Kronecker product form; 
$\K = \bigotimes_{l=1}^{d} \K_l$, 
where $\K_l \inR{m_l \times m_l}$ is a one-dimensional kernel covariance matrix along a slice of the $l^\text{th}$ input dimension, and $m_l$ is the number of points along the $l^\text{th}$ input dimension on the Cartesian product grid%
\footnote{We can also relax the dataset structure requirements by assuming that \emph{groups} of inputs form a grid. See the \S\ref{sec:climate} study.}~\cite{saatci_phd}.
Exploiting this Kronecker product structure, we find that  
only $\mathcal{O}(dN^{\frac{2}{d}})$ storage is required, and
a matrix-vector product with $\K$ requires only $\mathcal{O}(dN^{\frac{d+1}{d}})$ time%
\footnote{It is assumed that there are 
$m_1{=}m_2{=}\dots{=}m_d{=}\sqrt[\leftroot{0}\uproot{0}d]{M}$ points along each dimension  of the tensor product grid. 
Additionally, we refer to \cite{saatci_phd} for an efficient matrix-vector product algorithm.}.
However, when there are missing responses from the full grid, the Kronecker product structure is broken and $\Kxx \inR{N\times N}$ becomes a large dense matrix with no structure. 
As a result, GP modelling storage and time increase to the nominal complexities of \order{N^2} and \order{N^3}, respectively.
We shall next discuss some novel approaches for addressing this computational challenge.

\section{Linear Algebra aspects of GP Training}
\label{sec:weight_vec}
\lesslines
Here we propose new algorithms to compute a matrix-vector product with the covariance matrix inverse, 
$(\Kxx + \sigma^2 \I{N})^{-1}$,
which is the most computationally demanding component of performing exact GP regression.
We first review the existing ``penalize-gaps'' (PG) formulation of \citet{wilson_gpatt_nips} which requires solving an unstructured linear system of equations of size $M \times M$ (number of points on full grid).
We then proceed to develop two alternative formulations; 
the ``ignore-gaps'' (IG) method which reduces the size to $N \times N$ (number of training points); and
the ``fill-gaps'' (FG) method which only requires solving an unstructured linear system of size $L \times L$ (number of gaps).
We complete this section by developing a preconditioner for the two new formulations~presented.

\subsection{Gap Penalization Strategy}
\label{sec:PG}
\citet{wilson_gpatt_nips} first approached this problem when using GPs to model observations on a spatial-temporal Cartesian product grid.
The central idea of their solution was to fill in the gaps in the
observations with arbitrary values and subsequently use a penalty approach to ensure that these pseudo-observations do not influence the final model. 
This may appear to be somewhat counter-intuitive since they are
essentially increasing the number of observations, however, this step allows fast matrix-vector products to be made with the $M \times M$ covariance matrix $\K$ which has
a Kronecker product structure.
\begin{proposition}[Penalize-Gaps, PG]
The vector $\alp \in \mathbb{R}^M$ obtained from the numerical solution of
the penalized $M\times M$ system of equations 
\begin{equation}
\label{eqn:penalty}
\left( \K + \gamma {\bf R} + \sigma^2 \I{M} \right) \boldsymbol{\alpha} = {\bf y}
\end{equation}
satisfies $(\Kxx+\sigma^2 \I{N})\alp_\text{X} = \y_\text{X}$ in the limit of the penalty
parameter $\gamma \rightarrow \infty$, where 
$\K \inR{M \times M}$ is the kernel covariance matrix on the full product grid,
${\bf R} \in \mathbb{R}^{M\times M}$ is an all zero matrix except
$\mat{R}_\text{X,X} = \I{N}$, and arbitrary numerical values are
inserted in the missing entries of $\y \in \mathbb{R}^M$.
\end{proposition}
A proof is provided in the supplementary material of \cite{wilson_gpatt_nips}.
A  linear conjugate gradient (CG) solver \citep{atkinson_txtbk} can be used to solve \cref{eqn:penalty} to take advantage of fast matrix-vector products in \order{dM^\frac{d+1}{d} {+} N} since
$\K {=} {\bigotimes_{i=1}^d} \K_i$ has a Kronecker product form.
The preconditioner 
$( \gamma {\bf R} {+} \sigma^2 \I{M} )^{-\frac{1}{2}}$ was suggested for this formulation \cite{wilson_gpatt_nips}.

Unfortunately, this method can suffer from numerical inaccuracies with a poor choice of $\gamma$, and it is not clear how large the penalty parameter $\gamma$ should be \emph{a priori}. 
Two novel approaches that do not suffer from this limitation are presented next.

\subsection{Selection Matrix Strategy}
\label{sec:IG}
\label{sec:ignore_gaps}
Using sparse selection matrices, we develop a technique to exploit Kronecker matrix algebra without the use of a penalty.
\begin{proposition}[Ignore-Gaps, IG]
\label{thm:ignore_gaps}
\label{thm:IG}
The solution of the linear algebraic system of equations 
$(\Kxx + \sigma^2 \I{N})\alp_\text{X} = \y_\text{X}$ 
also satisfies the system of equations 
\begin{align}
\W\left(\K + \sigma^2 \I{M}\right)\W^T \alp_\text{X} = \y_\text{X},
\end{align}
where $\W \inR{N \times M}$ is a sparse selection matrix with
one value per row set to unity in the column corresponding to each index in $\text{X}$ such that~$\W \K \W^T = \Kxx$.
\end{proposition}%
The proof of this formulation is trivial and it is evident that it also lends itself well to a conjugate gradient solver since it admits fast matrix-vector products in 
\order{dM^\frac{d+1}{d} {+} N} due to the form of  
 $\K {=} {\bigotimes_{i=1}^d} \K_i$. 
Evidently, this method does not require any user-defined parameters unlike the previous penalize-gaps method.

We observe that this covariance matrix structure coincides with that of ``structured kernel interpolation'' (SKI)~\cite{wilson_kiss} in the special case where training data coincides with the inducing point grid and so the SKI interpolation matrix becomes the sparse selection matrix $\mat{W}$.
We emphasize that while SKI typically uses an approximate kernel, if the inducing point grid and training data coincide, then the exact kernel is recovered during training.
At prediction time, the methods differ since the SKI kernel evaluated between train and test points becomes approximate whereas we continue to use the exact kernel.

\subsection{Fill Gaps Strategy}
\label{sec:FG}
\label{sec:fill_gaps}
While the previous two approaches determine $\alp_\text{X}$ directly, 
this approach first infers the posterior mean on the gaps, $\y_\text{Z}$, thus recovering a non-gappy problem upon which we can fully exploit Kronecker matrix algebra.

\begin{proposition}[Fill-Gaps, FG]
\label{thm:fill_gaps}
\label{thm:FG}
The solution of $(\K + \sigma^2 \I{M}) \alp = \y$ satisfies $(\Kxx + \sigma^2 \I{N}) \alp_\text{X} =  \y_\text{X}$ where 
the missing values, $\y_\text{Z}$, are determined by
\begin{multline} \label{eqn:FG}
 \V\Ksi\V^T \y_\text{Z}=\\-\V\Ksi\W^T \y_\text{X},
\end{multline}
where $\V \inR{L \times M}$ is a sparse selection matrix
that has
one value per row set to unity in the column corresponding to each index in $\text{Z}$, and
$\Q, \T \inR{M \times M}$ are unitary and diagonal matrices, respectively, formed from the eigendecomposition of $\K = \Q \T \Q^T$.
\end{proposition}
\begin{proof}
First, considering the special case $\sigma^2 {=} 0$, observe that we can use
$\alp_\text{X} {=} \Kxx^{-1} \y_\text{X}$ and $\alp_\text{Z} {=} 0$ 
to infer the posterior mean of the missing values on the gaps, $\y_\text{Z}$, by $\y {=} \K \alp$.
Here we attempt to solve for $\y_\text{Z}$ directly without first computing $\alp_\text{X}$.
We start by writing $\alp {=} \K^{-1}\y$ in partitioned form, and impose $\alp_\text{Z} {=} \0$
\begin{equation*}
\left[
\begin{array}{ll}
\Ki_\text{X,X} & \Ki_\text{X,Z}\\
\Ki_\text{Z,X} & \Ki_\text{Z,Z}
\end{array} 
\right] \left[
\begin{array}{l}
\y_\text{X}\\
\y_\text{Z}
\end{array}
\right] = \left[
\begin{array}{c}
\alp_\text{X}\\
\0
\end{array}
\right],
\end{equation*}
where $\Ki_{(\ \cdot\ ,\ \cdot\ )}$ is a partition of $\K^{-1}$.
Rearranging the last row gives
$\Ki_\text{Z,Z}\,\y_\text{Z} = - \Ki_\text{Z,X}\, \y_\text{X}$,
and observing that 
$\V \K^{-1} \V^T = \Ki_\text{Z,Z}$ and 
$\V \K^{-1} \W^T = \Ki_\text{Z,X}$
gives the system of equations
\begin{align} \label{eqn:FG_no_noise}
\V{\bf K}^{-1}\V^T \y_\text{Z}&= -\V{\bf K}^{-1}\W^T \y_\text{X}.
\end{align}
Next, if we consider $\sigma^2 {>} 0$, then we require
a partition of $(\K+\sigma^2 \I{M})^{-1} $ instead of $\K^{-1}$.
We can write this as
$(\K+\sigma^2 \I{M})^{-1} {=} \mat{Q}\left(\mat{T}+\sigma^2 \I{M}\right)^{-1}\mat{Q}^T$
using the eigendecomposition $\K {=} \Q \T \Q^T$ defined in the proposition statement which only requires  \order{dM^\frac{3}{d}} time to compute using Kronecker matrix algebra.
Substituting this for $\K^{-1}$ in \cref{eqn:FG_no_noise} completes the proof. \hfill \ensuremath{{\Box}}
\end{proof}

Similar to the other methods, this formulation 
lends itself well to a conjugate gradient solver since it admits fast matrix-vector products in 
 \order{dM^\frac{d+1}{d} {+} L} due to the structure of $\K {=} {\bigotimes_{i=1}^d} \K_i$. 
The eigendecomposition of $\K {=} \Q \T \Q^T$ can be rapidly computed in \order{dM^\frac{3}{d}} time using Kronecker matrix algebra, and as a result, 
$\Q {=} {\bigotimes_{i=1}^d} \Q_i$ is also a Kronecker product matrix \cite{van_loan_kron}.
Lastly, since $\mat{T}+\sigma^2 \I{M}$ is diagonal, matrix-vector products with its inverse cost only \order{M} time.
We would like to emphasize that the time complexity of matrix-vector products are
{\em independent} of the number of training points, $N$. 
Additionally, we observe that since the system being solved by a CG method is $L \times L$, the number of iterations required is expected to be much less than $L$, the number of gaps. 
We therefore expect that this method would be extremely fast for massive data-sets where there are few gaps
but in \cref{sec:experiments} we find that it also outperforms other methods well outside of this regime.
Once the missing values $\y_\text{Z}$ are inferred, $\alp_\text{X}$ can be found by evaluating
$\alp {=} \mat{Q}\left(\mat{T}+\sigma^2 \I{M}\right)^{-1}\mat{Q}^T\y$, which requires only \order{dM^\frac{d+1}{d}} time.

\subsection{Preconditioning strategies}
\label{sec:precon}
Using the ignore-gaps technique outlined in \cref{sec:IG}, and the eigendecomposition of $\K=\Q\T\Q^T$, observe that we can write the kernel covariance matrix as $\Kxx {=}\W \Q \T \Q^T \W^T$.
We can also approximate $\Kxx$ using the $p$ largest eigenvalues and corresponding eigenvectors of $\K$ as
\begin{align*}
\Kxx \approx \widetilde{\K}_\text{X,X}
&= \W \Q \mat{S}_p^T  
     \big(\mat{S}_p \T \mat{S}_p^T \big)
    \mat{S}_p \Q^T \W^T\\
&= \W \Q \mat{S}_p^T   {\T}_p \mat{S}_p \Q^T \W^T,
\end{align*}
where $\mat{S}_p \inR{p \times M}$ is a sparse selection matrix whose $i^\text{th}$ row has one value set to unity in the column corresponding to the index of the $i^\text{th}$ largest eigenvalue of $\K$ on the diagonal of $\mat{T}$; and 
${\T}_p \inR{p \times p}$ is a subset of $\T$.
We can then use the matrix inversion lemma to invert $\widetilde{\K}_\text{X,X} + \sigma^2 \I{N}$ 
\begin{multline} \label{eqn:IG_precon}
(\widetilde{\K}_\text{X,X} + \sigma^2 \I{N})^{{-}1} 
 =  \frac{1}{\sigma^{2}} \Big[
   \I{N} {-}  \\ 
   \W \Q \mat{S}_p^T    
   \big(
   \sigma^2 \I{p} {+}
   \T_p
   \mat{S}_p \Q^T \W^T
   \W \Q \mat{S}_p^T    
   \big)^{{-}1}
   \T_p
   \mat{S}_p \Q^T \W^T
   \Big],
\end{multline}
which only requires the additional storage and
inversion of a matrix of size $p \times p$.
After the $p\times p$ matrix has been factorized, multiplications with this preconditioner cost only \order{dM^\frac{d+1}{d} {+} p^2} time.

\lesslines
The use of $\widetilde{\K}_\text{X,X}$ for matrix preconditioning was explored with notable empirical success in \cite{cutajar_preconditioning} where a sub-set of training data was used as ``inducing points'' giving $M<N$, as opposed to a super-set as it is here ($M>N$) which we expect improves performance.

In \cref{sec:ignore_gaps}, we observed the duality between the ignore-gaps technique and SKI \cite{wilson_kiss}.
We would further like to point out that the preconditioner \cref{eqn:IG_precon} can similarly be used in the SKI framework, however we will not study its  effectiveness in a general SKI setting. 

Considering now a preconditioner for the fill-gaps method of \cref{sec:FG}, 
it can be shown that the matrix on the left-hand side of \cref{eqn:FG} can be rewritten as
\begin{multline*}
 \V\Ksi\V^T  = \mat{J} + \zeta \I{L},\\
= \V \Q \left[ (\T + \sigma^2\I{M})^{-1} - \zeta \I{M}\right]\Q^T\V^T  + \zeta \I{L}
\end{multline*}
where we have applied a spectral shift of $\zeta \inR{}$ to the first term, which we call $\mat{J} \inR{L \times L}$.
We can now write a rank-$p$ approximation of $\mat{J}$ as
\begin{align*}
\mat{J} \approx \widetilde{\mat{J}} &=
     \V \Q \widebar{\mat{S}}_p^T
     \left(\widebar{\mat{S}}_p \left[ (\T {+} \sigma^2\I{M})^{-1} {-} \zeta \I{M}\right] \widebar{\mat{S}}_p^T \right)
    \widebar{\mat{S}}_p \Q^T \V^T\\
&= \V \Q \widebar{\mat{S}}_p^T     \widebar{\T}_p \widebar{\mat{S}}_p \Q^T \V^T,
\end{align*}
where
$\widebar{\T}_p \inR{p\times p}$ is a subset of $\left[ (\T {+} \sigma^2\I{M})^{-1} {-} \zeta \I{M}\right]$ containing the largest values on its diagonal, and
$\widebar{\mat{S}}_p \inR{p \times M}$ is a sparse selection matrix whose $i^\text{th}$ row has one non-zero value set to unity in the column corresponding to the index of the $i^\text{th}$ smallest eigenvalue of $\K$ on the diagonal of ${\mat{T}}$.
To ensure no singularities, we choose $0 < \zeta < (\lambda_p + \sigma^2)^{-1}$ where $\lambda_p \inR{}$ is the $p^\text{th}$ smallest eigenvalue of $\K$ (in practice we choose a value mid-range).
$\widetilde{\mat{J}} + \zeta \I{L}$ is then an approximation of the fill-gaps left-hand side matrix which we can cheaply invert to use as a preconditioner with the matrix inversion lemma
\begin{multline} \label{eqn:FG_precon}
(\widetilde{\mat{J}} + \zeta \I{L})^{{-}1} 
 = \zeta^{-1}\Big[
   \I{L} - \\
   \V \Q \widebar{\mat{S}}_p^T    
   \big(
   \zeta \I{p} +
   \widebar{\T}_p
   \widebar{\mat{S}}_p \Q^T \V^T
   \V \Q \widebar{\mat{S}}_p^T    
   \big)^{{-}1}
   \widebar{\T}_p
   \widebar{\mat{S}}_p \Q^T \V^T
   \Big],
\end{multline}
which similarly admits fast matrix vector products in \order{dM^\frac{d+1}{d} {+} p^2} time.
We will analyze the efficacy of both of these preconditioners in \cref{sec:precon_studies}.

\section{Model Selection \& Fast Inferencing}
\label{sec:inference}
\lesslines
To compute the log marginal likelihood in \cref{eqn:likelihood},
we need
$(\Kxx + \sigma^2 \I{N})^{-1} \y_\text{X}$, and $\log |\Kxx + \sigma^2 \I{N}|$.
We already discussed efficient ways to compute the former and we will use a \nystrom\ approximation for the latter, which is accurate for large~$N$~\cite{wilson_gpatt_nips}; 
\begin{equation*}
\log |\Kxx + \sigma^2 \I{N}| 
\approx \sum \limits_{i=1}^N \log \left(\tfrac{N}{M}\lambda_i + \sigma^2\right),
\end{equation*}
where 
$\lambda_i$ is the $i^\text{th}$ largest eigenvalue of $\K {=} {\bigotimes_{j=1}^d} \K_j$ which can be rapidly computed in \order{dM^\frac{3}{d}} time due to its Kronecker product form \cite{van_loan_kron}.
We now have all the ingredients required to efficiently train our model; we can now compute and maximize the marginal likelihood with respect to the hyperparameters and compute $\alp_\text{X}$.
We shall next outline how to rapidly compute the posterior mean and covariance during inference.

We frequently need to infer test points distributed on a tensor product grid (e.g. for search or visualization).
Here we demonstrate how we can exploit Kronecker matrix algebra to evaluate the posterior mean of points on a grid extremely quickly.
From \cref{eqn:posterior}, the posterior mean for a single point is given by 
$\mat{g}_\text{X}^T\alp_\text{X}$ which is equivalent to 
${\mat{g}}^T \alp$, where $\alp_\text{X} = (\Kxx + \sigma^2 \I{N})^{-1} \mat{y}_\text{X}$ and $\alp_\text{Z} = \0$.
If it is sought to perform inference at many points then we can replace 
${\mat{g}}$ with 
${\mat{G}} \inR{M \times Q}$ which is formed by horizontally stacking the column vectors ${\mat{g}}$ for each of $Q$ distinct test points.
Further, if we take these test points to be on a Cartesian product grid then a Kronecker product structure is inherited,
${\mat{G}} {=} {\bigotimes_{i=1}^d} {\mat{G}}_i$, where ${\mat{G}}_i \inR{m \times q}$, and we take $Q {=} q^d$, $M {=} m^d$.
By recognizing and exploiting this structure,
the time required to compute the posterior mean at the $Q$ test points decreases from
\order{NQ} to \order{\sqrt[d]{Q}M {+} \sqrt[d]{M}Q}.
This can give significant computational advantages for large $Q$.

The posterior covariance computation poses a different problem, since, from \cref{eqn:posterior} it is evident that it requires the solution of an $N\times N$ system of equations for \emph{each} test point, $(\Kxx + \sigma^2 \I{N})^{-1}\mat{g}_\text{X}$.
We could compute this using any technique developed in \cref{sec:weight_vec}; 
extension of the penalize-gaps and ignore-gaps techniques for this problem is trivial, however, applying the fill-gaps formulation gives an interesting interpretation. 
What we are ``filling in'' here is $\mat{g}_\text{Z}$ which is actually the cross-covariance between gaps on the training grid and the test point such that the final solution is not influenced by the gaps.
Once $\mat{g}_\text{Z}$ is filled in, the fully structured problem can be solved rapidly to complete the posterior covariance computation.
Since this problem is effectively identical to the training problem, we refer to \cref{sec:experiments} for a comparison of the different computation methods, however, it must be considered that any time invested in formulating a preconditioner once can be used to perform inference at many test points.

We may alternatively consider an \emph{approximation} of the posterior covariance
where we make use of the matrix preconditioner developed in \cref{sec:precon} by replacing
$\Kxx$ with $\widetilde{\K}_\text{X,X}$ in \cref{eqn:posterior}.
This gives us a similar 
posterior distribution as the \nystrom\ method for GP modelling studied by \citet{williams_nystrom} and would enable the posterior covariance to be computed in 
\order{dM^\frac{d+1}{d} {+} p^2} time per test point.
We will not consider this approximation in the experiments of \cref{sec:experiments} but will instead restrict our attention to exact GP inference.

\section{Experiments}
\label{sec:experiments}
\subsection{Stress Tests}
\label{sec:stress_tests}
\begin{figure*}[ht!]
	\centering
	\begin{subfigure}[b]{0.3\textwidth}
		\includegraphics[width=\textwidth]{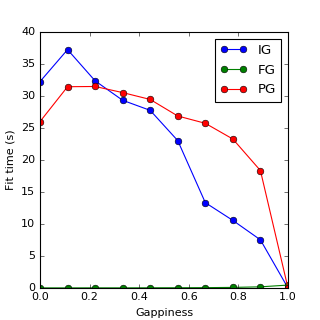}
		\caption{$M = 10,000$}
		\label{fig:10K}
	\end{subfigure}%
	\begin{subfigure}[b]{0.34\textwidth}
		\includegraphics[width=\textwidth]{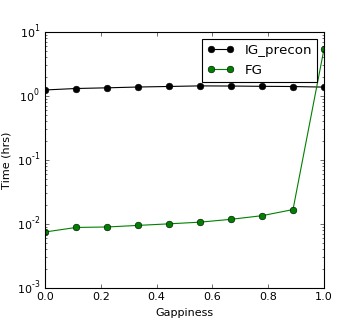}
		\caption{$M = 16,588,800$}
		\label{fig:10M}
	\end{subfigure}%
	\begin{subfigure}[b]{0.34\textwidth}
		\includegraphics[width=\textwidth]{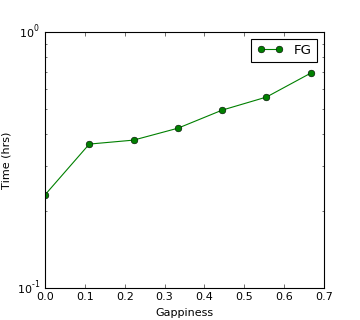}
		\caption{$M = 1,003,622,400$}
		\label{fig:1B}
	\end{subfigure}
	\caption{
		Reconstruction timings comparing the techniques outlined in \cref{sec:weight_vec} across a range of gappiness
		on various problem sizes, $M$.
		The three techniques are 
		FG~-~fill-gaps with no preconditioner (\cref{sec:FG}), 
		IG~-~ignore-gaps with no preconditioner (\cref{sec:IG}),
		IG\_precon~-~ignore-gaps with a rank-5000 preconditioner (\cref{sec:precon}) and
		PG~-~penalize-gaps with the preconditioner discussed in \cref{sec:PG}.
		The timings include both the time to fill in the missing values, as well as compute $\K^{-1}\y = \alp$.
		\\ \vspace{\lessfigspace} \mbox{} 
	}
	\label{fig:gappiness}
\end{figure*}
\lesslines
We test the robustness of techniques outlined in
\cref{sec:weight_vec} for training a GP regression model on massive datasets of varying ``gappiness''${=}(M{-}N)/M{=}L/M$.
Results are shown in \cref{fig:gappiness} for gappiness sweeps across various grid sizes, $M$.
For the $M{=}10,000$ case, data was generated from the two-dimensional Rastrigin function \cite{rastrigin} and for the larger studies, we reconstruct a gappy 4K video of a resonating elastic membrane defined by the two-dimensional wave equation,
$
\tfrac{\partial^2 y}{\partial x_1^2} + \tfrac{\partial^2 y}{\partial x_2^2} {=}
\tfrac{\partial^2 y}{\partial x_3^2},
$
where $x_1, x_2 \inR{}$ are spatial coordinates, and
$x_3 \inR{}$ is time.
The membrane is constrained along the edges of the video frame and begins with random initial conditions.
Our GP models use the squared-exponential kernel, $k(\mat{x},\mat{z}) = \exp({-}||\mat{x}{-}\mat{z}||_2^2 / \theta_\ell^2)$, and hyperparameter tuning is performed beforehand on the fully structured dataset.
We do not consider hyperparameter selection on the gappy data, since the computation of $\alp$ alone allows us to contrast differences between the considered techniques.
Gaps are randomly applied to mask the training data and
$\alp$ is computed using a CG solver to a tolerance of $10^{-6}$.  
All experiments use the authors' code%
\footnote{\url{https://github.com/treforevans/gp_grid}}
on a machine with two E5-2680 v3 processors and 128GB RAM.

From \cref{fig:10K} we can see that 
the non-preconditioned ignore-gaps (IG) and the preconditioned
penalize-gaps (PG) algorithm performed comparably, however,
the fill-gaps (FG) method significantly outperformed the others across much of the gappiness spectrum.

In the larger, video reconstruction problem in \cref{fig:10M},
IG and PG are not shown, since after a wall-clock time
of 5 days, only the $N{=}1$ case had converged.
However, the FG method and the ignore-gaps algorithm with a
rank-5000 preconditioner (IG\_precon) completed reconstructions
across the full gappiness spectrum with each run taking $\lessapprox 1$ hour.
We emphasize that although $p{=}5000$ is four orders
of magnitude smaller than $M$, 
the preconditioner dramatically accelerated convergence.

Moving to the 1 billion point video reconstruction problem in
\cref{fig:1B}, the IG\_precon method was omitted since the time
required to form the preconditioner became prohibitive, however , 
the FG routine completed all reconstructions up to 70\% gappiness in
under 1 hour.

In figures \ref{fig:10K} \& \ref{fig:10M}, we observe that the fill-gaps technique is fastest where the gappiness is low, however, the ignore-gaps technique is only faster when the gappiness is surprisingly close to one;
in \cref{sec:precon_studies} we provide insight into why the FG method converges so quickly.
 
\subsection{Preconditioner Studies}
\label{sec:precon_studies}
\begin{figure*}[t!]
	\centering
	\begin{subfigure}[b]{0.33\textwidth}
		\includegraphics[width=\textwidth]{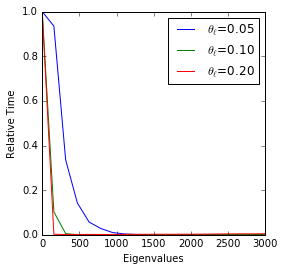}
		\caption{Ignore-Gaps CG Time}
		\label{fig:IG_precon}
	\end{subfigure}%
	\begin{subfigure}[b]{0.33\textwidth}
		\includegraphics[width=\textwidth]{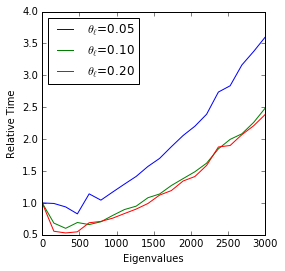}
		\caption{Fill-Gaps CG Time}
		\label{fig:FG_precon}
	\end{subfigure}%
	\begin{subfigure}[b]{0.33\textwidth}
		\includegraphics[width=\textwidth]{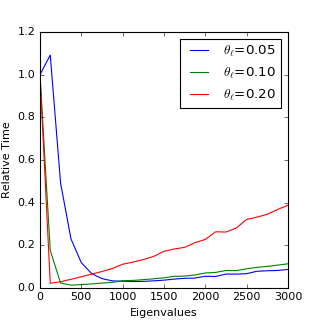}
		\caption{Ignore-Gaps Setup \& CG Time.}
		\label{fig:IG_incl_construction}
	\end{subfigure}
	\caption{
		Results of a preconditioner efficacy study. 
		We use a linear conjugate gradient solver to find $\alp_\text{X}$ using the preconditioner \cref{eqn:IG_precon} with varying values of the number of eigenvalues included, $p$, and kernel lengthscales, $\theta_\ell$.
		The timings in \cref{fig:IG_precon} and \cref{fig:FG_precon} only include the time required to solve the linear system using a CG solver while
		in \cref{fig:IG_incl_construction} the time required to construct the preconditioner is also included.
		All timings are presented relative to the time to perform the reconstruction with no preconditioner.
		\\ \vspace{\lessfigspace} \mbox{} 
	}
	\label{fig:preconditioner}
\end{figure*}
Here we assess the efficacy of the preconditioners outlined in \cref{sec:precon} for both the ignore-gaps and fill-gaps methods. 
For these tests, we take the training data to be structured on a two-dimensional grid with 100 points evenly spaced in the range $[0,1]$ along each dimension, giving $M = 10,000$.
We then randomly remove 50\% of the training points and sample responses as $\mat{y} \sim \mathcal{N}(\0, \I{M})$.
For our GP model,
we set $\sigma^2=10^{{-}6}$, and
use a squared-exponential kernel, 
considering a range of kernel lengthscales, $\theta_\ell$.
We then compute $\alp_\text{X} = (\Kxx + \sigma^2 \I{N})^{-1} \mat{y}_\text{X}$ using various values of preconditioner rank, $p$, and average over several samples of $\y$ and randomly applied gaps.
Results are shown in \cref{fig:preconditioner}.

Firstly, observe that the ignore-gaps scheme benefits dramatically through the use of a preconditioner (\cref{eqn:IG_precon}), which is consistent with the results of the stress tests.
\Cref{fig:IG_precon} shows the $\alp_\text{X}$ computation time \emph{excluding} the preconditioner setup time where it is evident that massive computational savings can be realized. 
This is representative of posterior covariance computations since we can reuse the preconditioner for many test points after an initial setup (discussed in \cref{sec:inference}).
The timings in \cref{fig:IG_incl_construction} \emph{include} setup time which is representative of using the preconditioner for training.
Here a broad minima is evident, indicating that an acceptable value of $p$ can be easily chosen.
We also see that as $\theta_\ell$ decreases, a greater number of eigenvalues, $p$, are required to achieve equivalent savings, as expected.

In \cref{fig:FG_precon} we analyze the fill-gaps preconditioner (\cref{eqn:FG_precon}) where it appears to only be effective for a very particular value of $p$, even when we exclude the preconditioner setup time.
If the setup time is included, we find that the method converges fastest with no preconditioner.
We expect this is because the eigenvalues of $(\K + \sigma^2 \I{M})^{-1}$ tend to decay slowly from the largest possible value of $\sigma^{-2}$ and so
the reduced-rank approximation used in the preconditioner is inaccurate.
We would also expect this eigenspectrum to make the non-preconditioned fill-gaps method particularly well suited to a CG solver, which is consistent with the fast convergence observed in the stress tests of \cref{sec:stress_tests}.

\subsection{Ontario Weather Stations}
\label{sec:climate}
\begin{figure*}[ht!]
	\includegraphics[width=\textwidth]{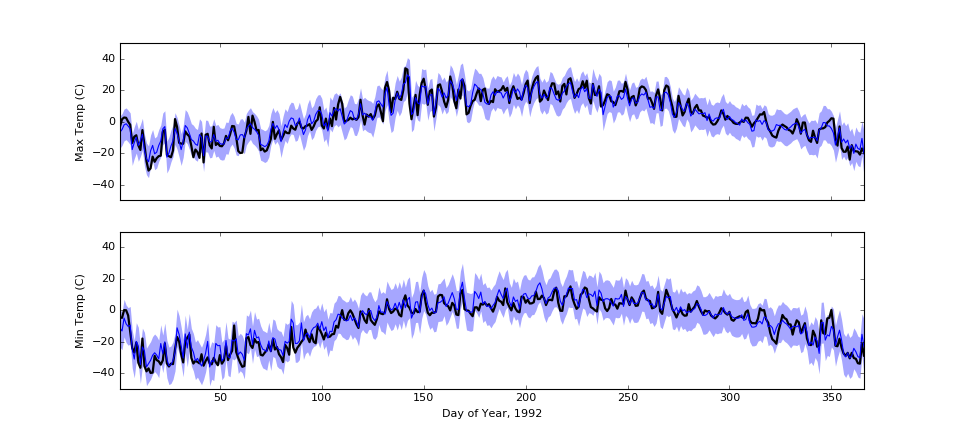}
	\caption{Reconstructed daily temperature observations for Moosonee in 1992.
		The black curves show the actual daily maximum (top) and minimum temperatures (bottom) which were both withheld from the model to compute the blue posterior distribution where the mean and three standard deviations (99.7\% confidence) are illustrated.
	\\ \vspace{\lessfigspace} \mbox{} 
	}
	\label{fig:moosonee}
\end{figure*}
Here we demonstrate an efficient and powerful analysis of weather patterns at 291 weather stations across Ontario every day from 1950 to 2005.
The dataset~\cite{ontario_climate} contains many missing observations, mostly due to corrupted recordings or because some weather stations were not in operation for the entire period.
Reconstructing these observations are important for environmental studies and typically this is done by interpolating responses spatially for each day, independent of all other days, and all other responses~\cite{canada_weather_interp,ontario_climate}.
These dependencies are ignored because the problem size becomes prohibitive when temporal or multi-output correlations are considered.
We identify structure in this problem, allowing us to consider correlations within space, time, and between responses to allow forecasting into the future with an accurate posterior distribution.
This ability to forecast is not possible with existing techniques.

We use a kernel for our multi-output GP which admits the following non-gappy covariance matrix;
\begin{align*}
\K = \K_\text{year} \otimes \K_\text{day} \otimes \K_\text{space} \otimes \mat{B},
\end{align*}
where 
$\K_\text{year} \inR{56 \times 56}$ is the covariance between the integer year of observations ($\in \{1950,\dots,2005\}$),
$\K_\text{day}  \inR{366 \times 366}$ is the covariance between the day of year of observations ($\in \{1,\dots,366\}$),
$\K_\text{space}  \inR{291\times 291}$ is the covariance between the location of the weather stations ($\inR{3},\ \{\text{latitude, longitude, elevation}\}$), and
$\mat{B} \inR{2\times2}$ is the covariance between the daily maximum and minimum temperature, which are the two observations we are modelling each day (see \cite{lawrence_multi_output_kernels} for multi-output kernel details).
We also apply a periodic transformation to the day-of-year kernel with a period of 365.25 days~\cite{osborne_gp_timeseries}.
Although the weather stations themselves are not distributed on a grid, we have evidently identified significant structure to exploit. 

The full grid size is $M{=}11,928,672$, 
however, over 6.5 million points are missing and 30\% of the remaining points are randomly withheld for testing, giving $N=3,742,547$ training points; an enormous problem for exact GP modelling.

\Cref{tbl:method_comparison} illustrates the results of the exact GP model constructed on the climate dataset using the different techniques outlined.
Model training includes hyperparameter estimation through marginal likelihood maximization and the root mean squared error (RMSE) evaluated on the randomly withheld test set is reported for daily minimum and maximum temperature predictions from the multi-output GP. 
The test error is quite low and is almost constant down the columns as we would expect since the different training methods should result in identical models. 
It is firstly evident that the penalize-gaps technique (PG) is the slowest method considered, even for a small penalty parameter ($\gamma{=}100$).
The non-preconditioned ignore-gaps (IG) technique trained about 22\% faster, however, we see that the training time can be further reduced by nearly an order of magnitude using a preconditioner with only $p{=}3000$.
The fill-gaps (FG) technique continues to be our most robust method, nearly halving the training time of any other.

\begin{table}[t]
	\centering
	\begin{tabular}{@{}llll@{}}
		\toprule
		 & Run Time & \multicolumn{2}{c}{RMSE $(^\circ C)$}\\ \cline{3-4}
		&\multicolumn{1}{c}{(hrs)} &Minimum & Maximum \\
		\midrule
		FG 		&\bf 11.5 & 2.02 & 1.45 \\
		IG 		& 173.1   & 2.02 & 1.45\\
		$\text{IG}_{1000}$ 	&37.3 & 2.02& 1.45\\	
		$\text{IG}_{3000}$ 	&19.7 & 2.02& 1.45\\	
		$\text{PG}_{100}$ & 221.5 & 2.02 & 2.02 \\
		\bottomrule
	\end{tabular}
	\caption{Reconstruction time and accuracy of daily maximum and minimum temperatures on the withheld test set using different training techniques for the multi-output GP. 
		$\text{PG}_{\#}$ means the penalty $\gamma{=}\#$ was used and
		$\text{IG}_{\#}$ means a rank $p{=}\#$ preconditioner was used.
		\\ \vspace{\lessfigspace} \mbox{} 
	}
	\label{tbl:method_comparison}
\end{table}

\Cref{fig:moosonee} compares the posterior distribution to the observed data at the Moosonee weather station in 1992.
Moosonee was chosen since it is the farthest from its nearest neighbour, the Smoky Falls weather station which is 171km inland~(see \cref{fig:temp_variance} for a map).
The GP does a very good  reconstruction of the withheld observations, especially considering the evident non-stationary behaviour.
In the shaded region, the posterior distribution is also shown with respect to each output.

\begin{figure}[t]
\includegraphics[width=0.5\textwidth]{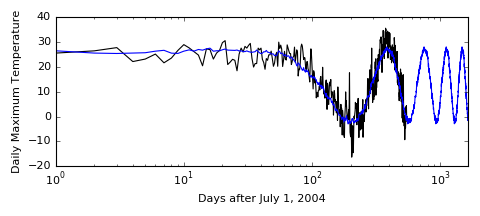}
\caption{Forecast of Toronto maximum daily temperature.
	Actual observations are in black and the posterior mean is in blue.
	Note that the x-axis is on a log scale.
	\\ \vspace{\lessfigspace} \mbox{} 
	}
\label{fig:forecast}
\end{figure}

\Cref{fig:forecast} shows a multiple year forecast at the Toronto weather station where it is evident that the posterior mean is quite reasonable.
For training, all data after July 1, 2004 was withheld along with all data at the Toronto weather station back to 1950.
In this way the forecast makes full use of the spatial-temporal correlations,
demonstrating the ability to forecast at a location where there is no historical data.
\Cref{sec:additional_results} provides additional results of the climate study.

\section{Concluding Remarks}
In this paper, we propose two novel and scalable exact GP regression algorithms for
massive datasets on a partial grid. 
Both our algorithms make extensive use of Kronecker matrix algebra
and we present novel preconditioners to accelerate computations. 
The proposed algorithms have modest memory requirements which enable us
to showcase performance on a real-world climate modelling problem with over 3.7~million training points
and on a synthetic video dataset problem with over 1 billion training points.
To the best of our knowledge exact GP inference has not been attempted before on this scale.

Compared to the penalize-gaps strategy of \citet{wilson_gpatt_nips}, both of the developed formulations eliminate the need for any user defined parameters and both of them ultimately decrease the size of the problem to be solved through the use of sparse selection matrices, rather than a penalty approach.
The fill-gaps technique from \cref{thm:FG} demonstrates a particularly interesting formulation wherein the gaps are ``filled-in'' \emph{before} explicitly training the model.
Additionally, both of the two developed techniques, fill-gaps and ignore-gaps, complement each-other elegantly where the first is solving a problem whose size is equal to the number of gaps while the latter is solving a problem size equal to the number of training points.
Thus, an appropriate method can always be selected based upon where the dataset lies in the ``gappiness'' spectrum.
Further, we developed a \nystrom-like preconditioner which is extremely effective at accelerating computations.

Lastly, while the methods discussed assume some dataset structure, we demonstrate through the climate modelling problem that with a little insight, structure can be found in many applications to enable exact GP modelling on massive datasets.

{
	\small
	\paragraph{Acknowledgements:} \hspace{-2.4mm} Research funded by an NSERC Discovery Grant and the Canada Research Chairs~program.
}
\\ \mbox{}\vspace{-5mm}
\dobib
\newpage
\appendix
\section{Ontario Climate Studies Additional Results}
\label{sec:additional_results}
This section provides some additional results from the Ontario climate studies of \cref{sec:climate}.
\Cref{fig:temp_variance} shows the log posterior variance of daily temperatures across Ontario.
Also shown is the scattered distribution of weather stations across the province.
As expected, the posterior variance rises away from the weather stations where the lack of data is reflected in our uncertainty.

\FloatBarrier
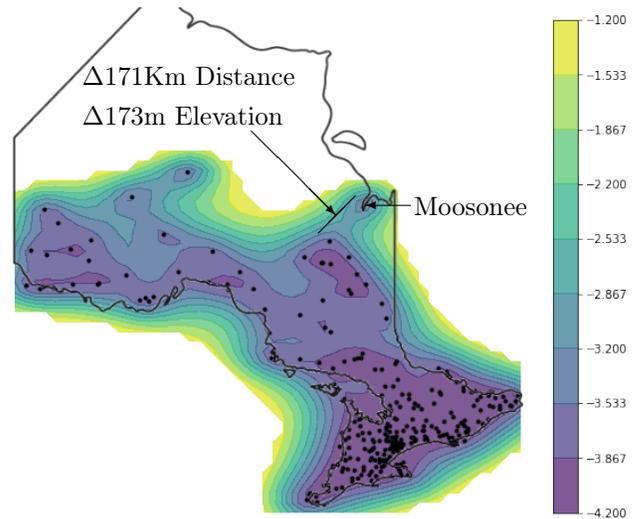
\begin{figure}
\centering
\begin{overpic}[width=0.5\textwidth]{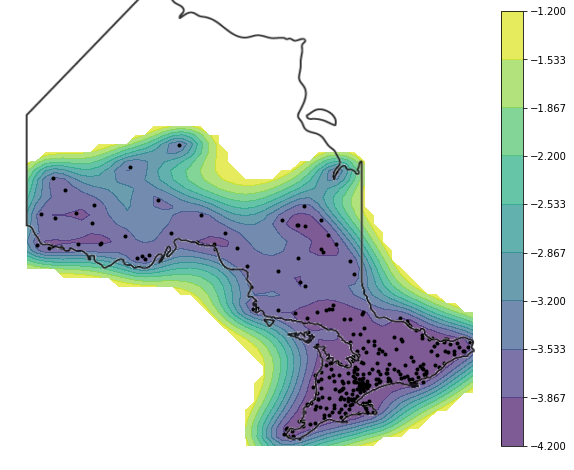} 
\put(66,50){\color{black}Moosonee}
\put(65.5,51.5){\color{black}\vector(-1,0){7}}

\put(15,64){\color{black}$\Delta$173m Elevation}
\put(15,70){\color{black}$\Delta$171Km Distance}
\put(57,53){\color{black}\line(-1,-1){5.7}}
\put(41,63){\color{black}\vector(1,-1){13}}
\end{overpic}
\caption{\lesslines Log posterior variance of daily temperatures across Ontario.
Black dots indicate weather station locations.
Moosonee is indicated along with the distance and elevation change between its nearest neighbour, the Smoky Falls weather station.
Temperature units are in degrees Celsius~($\mbox{}^\circ C$).}
\label{fig:temp_variance}
\end{figure}

\Cref{fig:time_kernel} shows the learned temporal kernel for the Ontario climate studies, found by maximum likelihood estimation. 
This stationary and decaying periodic kernel has a spike in correlation once every year indicating that the temperature on a given day is correlated to the same day of year the following years.
On a daily scale, it is evident that the kernel amplitude decays very rapidly suggesting that daily temperature is only correlated a couple days into the future and past.
On a yearly scale, it is evident that the kernel amplitude decays noticeably the first year and then essentially remains constant.
This means that the temperature on a given day in 1950 is correlated to the same day of year 55 years later in 2005.
This is consistent with our understanding of an annual climate.

Note that the kernel in \cref{fig:time_kernel} is very different from the learned temporal kernel presented in \cite{wilson_gpatt_nips} for a climate model constructed using pre-interpolated weather data on a Cartesian grid.
While we initialized our temporal kernel to be similar to the kernel of \citet{wilson_gpatt_nips}, it fit the data poorly and the maximum likelihood procedure rapidly moved towards a different structure.
The kernel in \cref{fig:time_kernel} is also more consistent with our prior understanding of the climate.
Perhaps the kernel presented in \cite{wilson_gpatt_nips} was stuck in a local optimum of their maximum likelihood procedure.

\begin{figure}
\includegraphics[width=0.5\textwidth]{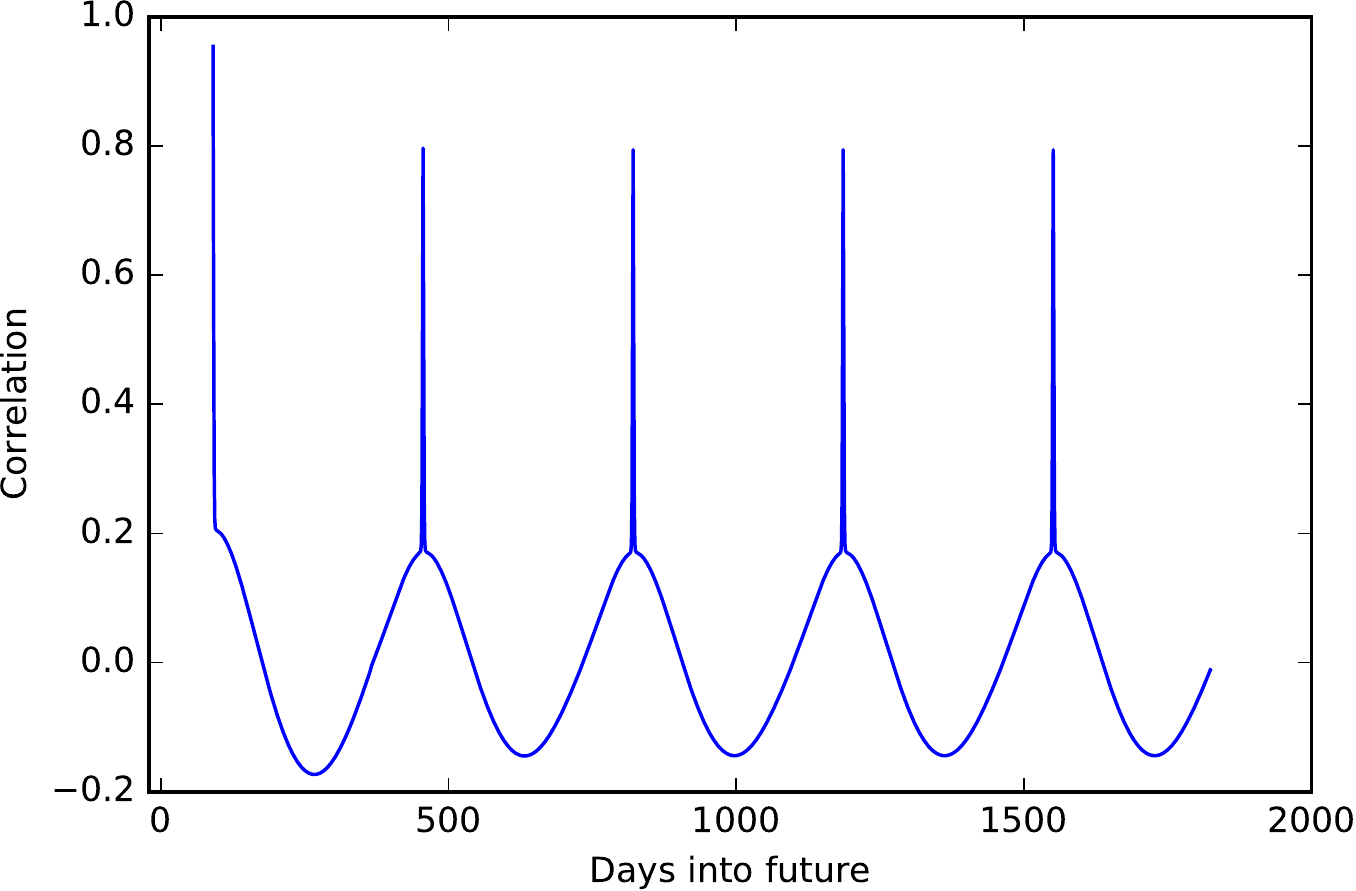}
\caption{Learned temporal kernel for the climate studies of \cref{sec:climate}.
This stationary and decaying periodic kernel has a spike in correlation once every year.
}
\label{fig:time_kernel}
\end{figure}

\section{Source Code}
Source code that implements the methods discussed in the paper along with several tutorials can be found at \url{https://github.com/treforevans/gp_grid}.
All the source code is written in Python.
\end{document}